\documentclass[letterpaper, 10pt, conference]{ieeeconf}

\IEEEoverridecommandlockouts 

\usepackage{enumerate}
\usepackage{graphicx}
\usepackage{amsmath, amssymb}

\usepackage{framed}
\usepackage{url}

\newtheorem{lemma}{Lemma}

\newtheorem{problem}{Problem}

\usepackage{ifthen,version}
\newboolean{include-notes}
\setboolean{include-notes}{true}
\newcommand{\tcnote}[1]{\ifthenelse{\boolean{include-notes}}%
 {\textbf{TC says: #1}}{}}

\newboolean{show-removed}
\setboolean{show-removed}{false}
\usepackage{color}
\newcommand{\removed}[1]{\ifthenelse{\boolean{show-removed}}%
  {{\color{red} #1}}{}}

\newboolean{show-proofs}
\setboolean{show-proofs}{true}
\newcommand{\showproof}[1]{\ifthenelse{\boolean{show-proofs}}%
  {#1}{}}

\DeclareMathAlphabet{\mathpzc}{OT1}{pzc}{m}{it}
\DeclareFontFamily{U}{msb}{}
\DeclareFontShape{U}{msb}{m}{n}{ <5> <6> <7> <8> <9> gen * msbm
<10> <10.95> <12> <14.4> <17.28> <20.74> <24.88> msbm10}{}
\DeclareSymbolFont{AMSb}{U}{msb}{m}{n}
\DeclareMathSymbol{\Reals}{\mathalpha}{AMSb}{'122}
\DeclareMathSymbol{\Naturals}{\mathalpha}{AMSb}{'116}
\DeclareMathSymbol{\Knumbers}{\mathalpha}{AMSb}{'113}
\DeclareMathSymbol{\Rationals}{\mathalpha}{AMSb}{'121}
\DeclareSymbolFont{AMSb}{U}{msb}{m}{n}
\DeclareMathSymbol{\setB}{\mathalpha}{AMSb}{'102}
\DeclareMathSymbol{\setC}{\mathalpha}{AMSb}{'103}
\DeclareMathSymbol{\setD}{\mathalpha}{AMSb}{'104}
\DeclareMathSymbol{\setE}{\mathalpha}{AMSb}{'105}
\DeclareMathSymbol{\setF}{\mathalpha}{AMSb}{'106}
\DeclareMathSymbol{\setI}{\mathalpha}{AMSb}{'111}
\DeclareMathSymbol{\setK}{\mathalpha}{AMSb}{'113}
\DeclareMathSymbol{\setM}{\mathalpha}{AMSb}{'115}
\DeclareMathSymbol{\setN}{\mathalpha}{AMSb}{'116}
\DeclareMathSymbol{\setP}{\mathalpha}{AMSb}{'120}
\DeclareMathSymbol{\setQ}{\mathalpha}{AMSb}{'121}
\DeclareMathSymbol{\setR}{\mathalpha}{AMSb}{'122}
\DeclareMathSymbol{\setS}{\mathalpha}{AMSb}{'123}
\DeclareMathSymbol{\setT}{\mathalpha}{AMSb}{'124}
\DeclareMathSymbol{\setU}{\mathalpha}{AMSb}{'125}
\DeclareMathSymbol{\setV}{\mathalpha}{AMSb}{'126}
\DeclareMathSymbol{\setW}{\mathalpha}{AMSb}{'127}
\DeclareMathSymbol{\setX}{\mathalpha}{AMSb}{'130}
\DeclareMathSymbol{\setY}{\mathalpha}{AMSb}{'131}
\DeclareMathSymbol{\setZ}{\mathalpha}{AMSb}{'132}

\title{Optimal Parameter Identification for Discrete Mechanical Systems with Application to Flexible Object Manipulation}

\author{T. M. Caldwell and  D. Coleman and N. Correll
\thanks{
This work was supported by a NASA
Early Career Faculty fellowship NNX12AQ47GS02. We are grateful for this support.}%
\thanks{T. M. Caldwell D. Coleman and N. Correll are with the Department of Computer Science, University of Colorado, 1111 Engineering Dr, Boulder, CO 80309, USA {\tt\small E-mail: caldwelt@colorado.edu ; david.t.coleman@colorado.edu ; nikolaus.correll@colorado.edu}}%
}

\begin{document}
\maketitle

\begin{abstract}
We present a method for system identification of flexible objects by measuring forces and displacement during interaction with a manipulating arm. We model the object's structure and flexibility by a chain of rigid bodies connected by torsional springs. Unlike previous work, the proposed optimal control approach using variational integrators allows identification of closed loops, which include the robot arm itself. This allows using the resulting models for planning in configuration space of the robot. In order to solve the resulting problem efficiently, we develop a novel method for fast discrete-time adjoint-based gradient calculation. The feasibility of the approach is demonstrated using full physics simulation in trep and using data recorded from a 7-DOF series elastic robot arm.
\end{abstract}

\section{Introduction}
The goal of this work is to use a robotic arm to identify the behavior of a flexible object through touch only. This is an important step toward manipulation of flexible objects such as rubber tubes, plants or clothes \cite{wakamatsu2006knotting,saha2007manipulation,bell2010flexible,jimenez2012survey}.  
There are many methods to model and simulate flexible objects \cite{khalil_payeur, lang_etal}.  A common approach is to model the object as a lattice or collection of links of masses and springs \cite{sahari_etal, wakamatsu_etal, khalil_payeur}.  This approach has been used to simulate linear object like strings, hair, and electrical cables for which the model is a series of masses linked together with springs. 

Our approach is similar with the primary difference that the loop connects back onto itself.  This connection restricts the loops movement and is handled using holonomic constraints.  To the best of our knowledge, no one has used a robot to identify parameters of a flexible loop for manipulation.

We are using a flexible loop as a running example throughout the paper.  We assume the robot has already rigidly grasped the object at one end and that it is clamped at the other.  The robot then bends, twists, and stretches the loop.  During the manipulation, the robot measures the arm's joint torques and joint angles.  With this information, it is possible to back out mechanical properties of the loop in order to generate an accurate model for future control and manipulation.  In this paper, the manner in which we model the loop is so that the underlying mechanics of the loop are the same as the robot arm, i.e., a collection of rigid bodies connected by springs, allowing us to utilize the vast theory of rigid body mechanics \cite{murray_li_sastry}. Also,  this enables planning and control to be done in the combined arm and loop configuration space instead of the end effector or object space.  We then use an optimal control approach for calculating model properties that best match the behavior of the physical loop.


\begin{figure}
\centering
\includegraphics[width = 180pt]{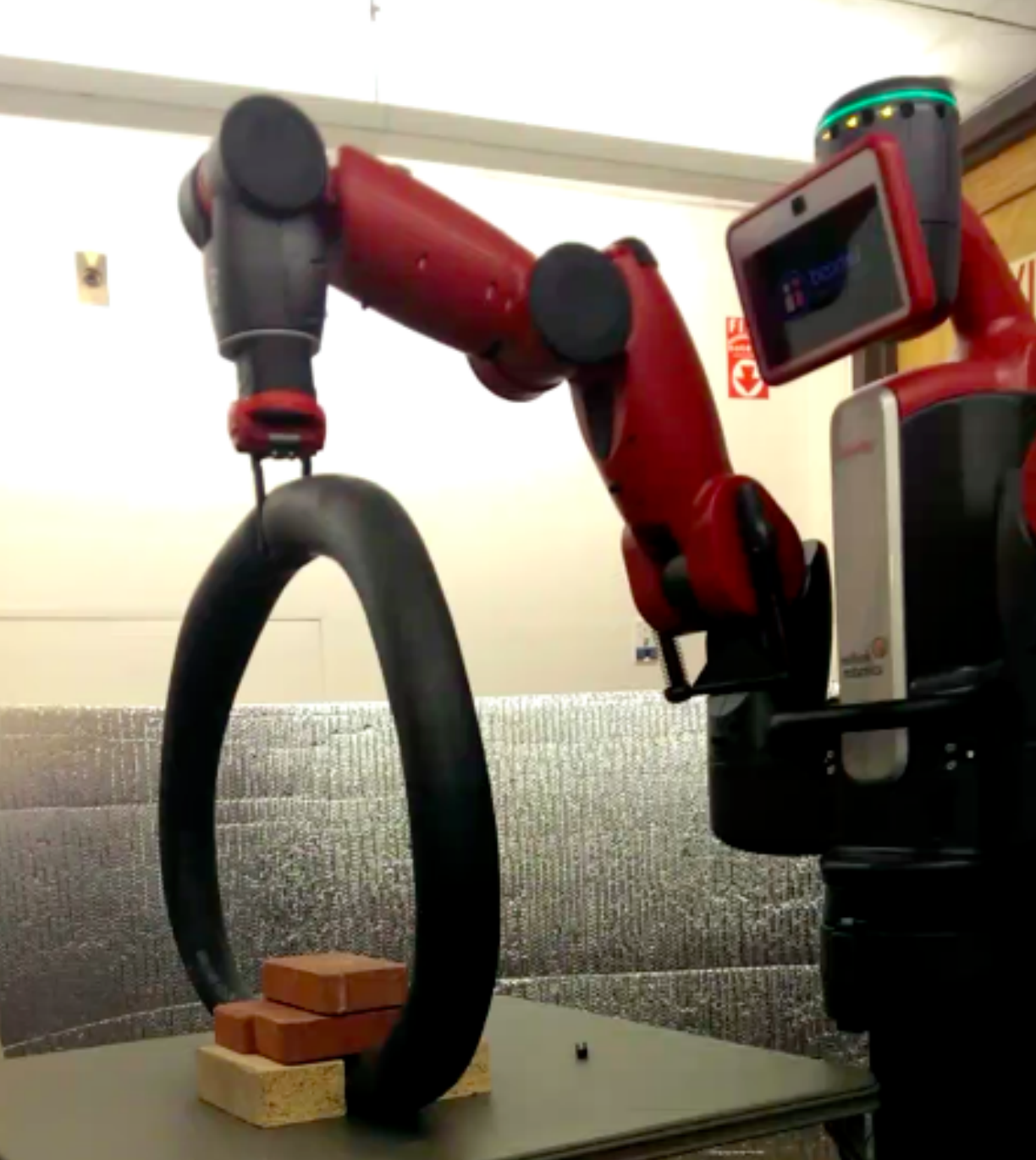}
\caption{Rethink Robotics' Baxter manipulating an inflated bicycle tyre.}
\label{fig-baxter_image_1}
\end{figure}

We use Rethink Robotics' Baxter \cite{guizzo2011rethink} robot to both manipulate and measure the loop.  Each of Baxter's arms have 7 degrees of freedom.  The arms are designed for compliance since each joint has series elastic actuators that allows for force sensing and control.  A picture of Baxter manipulating rubber loop is in Figure~\ref{fig-baxter_image_1}.

A good representation of a flexible object is not enough to accurately model it.  We also need the model simulation to be consistent with the physical behavior of the loop.  We decided not to apply Euler integration or another low order integrator to the continuous dynamics, as is the case in \cite{sahari_etal}, because such integrators can introduce significant energy errors.  At worst these errors will destabilize the integration and at best compromise the model's energy dissipation \cite{johnson_murphey_scalable}.  Instead we decided to use \emph{variational integrators}.  Variational integrators can be used to describe discrete-time equations of motion of a mechanical system.  They are designed from the least action principle and have good properties that agree with known physical phenomenon like stable energy behavior \cite{pekarek_murphey}.

Furthermore, variational integrators elegantly handle holonomic constraints.  Holonomic constraints are specified as $h(q) = 0$, where $q$ is the system's configuration.  They are used to constrain positions and orientations.  For the example, we use holonomic constraints to ``close the loop''---i.e. to constrain the link at one end of the loop to the other end.  In simulating continuous dynamics, holonomic constraints are commonly handled with equivalent constraints on the acceleration that can creep due to numerical integration error.  In comparison, variational integrators apply holonomic constraints directly and do not have this issue (see \cite{johnson_murphey_scalable} for a discussion).

Due to recent work by Johnson and Murphey \cite{johnson_murphey_scalable, johnson_murphey_linearization}, it is possible to efficiently simulate mechanical systems using variational integrators in generalized coordinates. They provide a framework using a tree representation and caching that not only makes for efficient simulation---especially for articulated rigid bodies---but also efficient model-based calculations like linearizations about a trajectory.   In optimal controls, linearizations are needed for gradient calculations like the gradient calculation presented in this paper. 

With a model and simulation to predict the motion of a flexible object, unknown model quantities, referred to as parameters, can be identified.  These model parameters may be lengths, masses, or damping coefficients.  For the example, we assume the loop's stiffness is unknown.  As such, the goal is to obtain the spring constants of the torsional springs at the loop model's joints.  In this paper, this is done using an optimal control approach.

The parameter identification optimization problem is set up as a discrete-time Bolza problem.  For the loop example, the cost functional is a summation of the error between the simulated end effector position and the experimentally-measured end effector position.  The joint error directly correlates to the difference between the simulated loop displacement and the measured loop displacement, at least at the manipulator.  Alternatively, the cost can be given by a maximum likelihood estimate for which the cost is lower for parameters that correspond to the simulation that is most consistent with the measurement (see \cite{houska_etal}).  

\subsection{Contribution of this paper}
The paper's contributions are twofold.  First, it provides a discrete-time adjoint-based gradient calculation for optimal parameter estimation.  Second, it formulates model parameter identification of flexible objects with variational integrators for greater confidence in the simulations' consistency with known physical properties.

\subsection{Organization of this paper}
This paper is organized as follows: Section \ref{sec-sys} reviews continuous and discrete mechanical systems, discussing variational integrators and providing linearization calculations.  Section \ref{sec-opt} discusses the parameter identification optimization problem as well as provides an adjoint-based gradient calculation.  Section \ref{sec-experiment} provides the details of the physical experiment using Rethink Robotics' Baxter, in addition to specifics on the loop model and simulation.  Moreover, it applies the parameter identification approach to estimate stiffness properties of the rubber loop model.

\section{Mechanical Systems}
\label{sec-sys}
In order to identify model parameters of a mechanical system, we must first be able to specify equations for simulation as well as model based calculations like linearizations.  This section reviews continuous and discrete time mechanical systems, where the latter is with respect to variational integrators.  A continuous mechanical system depends on its system Lagrangian which is the difference of its kinetic and potential energies.  From the Lagrangian, the system's equations of motion can be derived from the Euler Lagrange or similar equations depending on whether the system's movement is constrained or is externally forced.  The discrete mechanical system's equations of motion are similarly derived.  For the experiment, these equations predict how the Baxter arm moves as well as how it manipulates and deforms the loop. 

The mechanical system depends on $n_\rho$ system parameters from the parameter space $\mathcal{P}\in\setR^m$. The mechanical system has $n_q$ generalized coordinates $q\in\setR^{n_q}$.  For continuous time, $q$ varies with the continuous variable $t$\textemdash e.g. $q(t)$\textemdash for $t$ in the interval $[0,t_f]$, $t_f>0$.  Likewise, for the discrete representation, $q$ varies with the discrete variable $k$\textemdash e.g. $q_k$\textemdash for $k$ in the set $\mathcal{K}:=\{0,1,\ldots,k_f\}$, $k_f>0$.  Each discrete time $k$ pairs with a continuous time, labeled $t_k$, where $t_0 = 0$, $t_{k_f} = t_f$ and $t_k<t_{k+1}$. This section presents the continuous and discrete dynamics dependent on the parameters $\rho\in\mathcal{P}$.  The continuous dynamics are provided for comparison with the discrete dynamics, in which the paper results are given.

\subsection{Continuous Mechanical System}
We will first review continuous mechanical systems for reference with discrete mechanical systems.  A mechanical system's evolution is given by the path of least action.  The system's action is 
\[
S = \int_0^{t_f}L(q(\tau),\dot{q}(\tau),\rho)d\tau
\]
where $L(q,\dot{q},\rho) := KE(q,\dot{q},\rho) - V(q,\rho)$, the system Lagrangian, is the difference of the system's kinetic energy, $KE$, with its potential energy, $V$.  For this paper, we assume that both kinetic and potential energies depend on system parameters $\rho\in\mathcal{P}$.  Possible parameters of $L$ include lengths, spring constants, and masses. 

As is common in mechanical systems, we wish to include external forces, $F_c(q,\dot{q},\rho,t)$.  This term is the total external forcing in generalized coordinates, which we also assume depends on the parameters $\rho\in\mathcal{P}$, such as damping coefficients.  By including the additional external force term $F_c$ to the action, the Lagrange d'Alembert principle finds that the continuous dynamics of the mechanical system are given by the forced Euler-Lagrange equations \cite{murray_li_sastry}:
\[
\frac{d}{d t}\frac{\partial}{\partial \dot{q}}L(q,\dot{q},\rho) - \frac{\partial}{\partial q}L(q,\dot{q},\rho) = F_c(q,\dot{q},\rho,t).
\]
Holonomic constraints can be enforced as external forces using Lagrange multipliers.  The $n_h$ holonomic constraint equations are given in the form $h(q,\rho)= [h_1,\ldots,h_{n_h}]^T(q,\rho) = 0$. These constraints are important as they allow for describing closed loops like the flexible loop in the example.

With the addition of the constraint, the forced Euler-Lagrange equations are \cite{murray_li_sastry}:
\begin{equation}
\begin{array}{c}
\frac{d}{d t}\frac{\partial}{\partial \dot{q}}L(q,\dot{q},\rho) - \frac{\partial}{\partial q}L(q,\dot{q},\rho) = F_c(q,\dot{q},\rho,t) + \frac{\partial}{\partial q}h^T(q,\rho)\lambda\\
\frac{\partial^2}{\partial q^2}h(q,\rho)\circ(\dot{q},\dot{q}) + \frac{\partial}{\partial q}h(q,\rho)\ddot{q} = 0,
\end{array}
\label{eq-cont_EL}
\end{equation}
where $\lambda(t)\in\setR^{n_h}$ are Lagrange multipliers.  For fixed parameters $\rho$, the system's evolution is solved from Eq.(\ref{eq-cont_EL}) for $q$, $\dot{q}$ and $\lambda$. 

Equation \ref{eq-cont_EL} can be transformed into first-order state space equations.  Define the continuous state as $x = [q,\dot{q}]^T$.  The state equations, dependent on the parameters $\rho\in\mathcal{P}$, are $\dot{x}(t) = f(x(t),\lambda(t),\rho,t)$ where $\ddot{q}$ is specified by the constrained, forced Euler-Lagrange equations.  In the state space representation, gradient and Hessian calculations with respect to parameters are given for parameter optimization in \cite{miller_murphey}.

\subsection{Discrete Mechanical System}
The discrete mechanical system is an approximation of its continuous counterpart.  For an initial configuration $q(0)$ and velocity $\dot{q}(0)$, the continuous configuration $q([0,t_f])$ is integrated from the forced Euler-Lagrange equations, Eq.(\ref{eq-cont_EL}).  The discrete analog to the forced Euler-Lagrange equations instead calculates the sequence $q_k\approx q(t_k)$ using a variational integrator approach \cite{johnson_murphey_scalable}.  

The discrete Lagrangian, labeled $L_d$, is an approximation of the action over a short time interval.  Instead of a velocity term, the discrete Lagrangian is defined by the current and next configurations, $q_k$ and $q_{k+1}$:
\begin{equation}
L_d(q_k,q_{k+1},\rho) \approx \int_{t_k}^{t_{k+1}}L(q(\tau),\dot{q}(\tau),\rho)d\tau.
\label{eq-Ld}
\end{equation}
The integration can be approximated with a quadrature like midpoint or trapezoidal rules.  Refer to \cite{johnson_murphey_scalable} for details of using midpoint rule, which we use in the example.  

Similarly, external forcing is included by approximating $F_c$ with the discrete left and right forces $F_d^-(q_k,q_{k+1},\rho,t_k,t_{k+1})$ and $F_d^+(q_k,q_{k+1},\rho,t_k,t_{k+1})$ using a quadrature.  In addition, the $n_h$ holonomic constraints $h(q_k,\rho)$ can be enforced with the $n_h$ Lagrange multipliers $\lambda_k$.  Given the discrete Lagrangian, discrete forces, and holonomic constraints the forced discrete Euler-Lagrange equations are \cite{johnson_murphey_scalable}:\footnote{The notation $D_i$ is the slot derivative of the $i^\textrm{th}$ argument.  For example, $D_2 L_d(q_{k-1},q_{k},\rho)$ is the partial of $L_d$ with respect to the second slot, $q_{k}$.}
\begin{equation}
\begin{array}{l}
D_2 L_d(q_{k-1},q_{k},\rho) + F_d^+(q_{k-1},q_{k},\rho,t_{k-1},t_{k}) \\\hspace{20pt}+ D_1L_d(q_k,q_{k+1},\rho) + F_d^-(q_k,q_{k+1},\rho,t_k,t_{k+1}) \\\hspace{20pt} - D_1h^T(q_k,\rho)\lambda_k= 0 \\
\hspace{0pt}h(q_{k+1},\rho) = 0.
\end{array}
\label{eq-disc_EL}
\end{equation}
These equations should be viewed as an implicit function on $q_{k+1}$.  For example, given consecutive configurations $q_0$ and $q_1$, the next configuration $q_2$ is found with a root solving operation on Eq.(\ref{eq-disc_EL}). Following, $q_3$ is obtained from $q_1$ and $q_2$ and so forth.  Whereas the continuous mechanical system is solved using integration, the discrete mechanical system is solved through recursive calls to root finding Eq.(\ref{eq-disc_EL}). 

As in \cite{johnson_murphey_scalable}, define $p_k$ as 
\begin{equation}
p_k := D_2 L_d(q_{k-1},q_{k},\rho) + F_d^+(q_{k-1},q_{k},\rho,t_{k-1},t_{k}).
\label{eq-pk}
\end{equation}
Without external forcing, $p_k$ is the conserved momentum.  The discrete state is $x_k := [q_k,p_k]^T$ which has a one-step mapping:
\begin{equation}
\begin{array}{l}
x_{k+1} = f(x_k,\rho,t_k):=\\\left\{\begin{array}{l}
p_k + D_1L_d(q_k,q_{k+1},\rho) + F_d^-(q_k,q_{k+1},\rho,t_k,t_{k+1}) \\\hspace{20pt} - D_1h^T(q_k,\rho)\lambda_k= 0 \\
h(q_{k+1},\rho) = 0\\
p_{k+1} = D_2 L_d(q_{k},q_{k+1},\rho) + F_d^+(q_{k},q_{k+1},\rho,t_{k},t_{k+1}).
\end{array}\right.
\end{array}
\label{eq-fk}
\end{equation}
This equation is the state equation for discrete mechanical systems using variational integrators.  The function $f(x_k,\rho,t_k)$ is implicit, but, according to the Implicit Function Theorem, it exists when
\[
\begin{array}{l}
M_{k+1}:=D_2D_1L_d(q_k,q_{k+1},\rho) \\\hspace{50pt}+ D_2F_d^-(q_k,q_{k+1},\rho,t_k,t_{k+1})
\end{array}
\]
is nonsingular.  By assuming nonsingular, even though $f$ is implicit, the linearization around $x_{k+1}$\textemdash i.e. $\frac{\partial x_{k+1}}{\partial x_k}$\textemdash is explicit.  Letting $dx_k = [dq_k,dp_k]^T$ be the differential of $x_k$ and $da$ be the differential of $\rho$, the linearization of $f(x_k,\rho,t_k)$ is:
\begin{equation}
\left[\begin{array}{cc}dq_{k+1} \\ dp_{k+1} \end{array}\right]
 = \underbrace{\left[\begin{array}{cc}
\frac{\partial q_{k+1}}{\partial q_k} & \frac{\partial q_{k+1}}{\partial p_k} \\
\frac{\partial p_{k+1}}{\partial q_k} & \frac{\partial p_{k+1}}{\partial p_k} 
\end{array}\right]}_{A_k}
\left[\begin{array}{cc}dq_{k} \\ dp_k \end{array}\right] + \underbrace{\left[\begin{array}{cc}\frac{\partial q_{k+1}}{\partial a} \\ \frac{\partial p_{k+1}}{\partial a} \end{array}\right]}_{B_k}da
\label{eq-lin_fk}
\end{equation}
The calculations for the linearization term $A_k$ is given in \cite{johnson_murphey_linearization} and duplicated here for reference:
\begin{subequations}
\label{eq-A}
\begin{equation}
\begin{array}{l}
\frac{\partial q_{k+1}}{\partial q_k} =\\\hspace{0pt}  -M_{k+1}^{-1}[D_1^2L_d(q_k,q_{k+1},\rho) + D_1F_d^-(q_k,q_{k+1},\rho,t_k,t_{k+1}) \\\hspace{0pt} -D_1^2h^T(q_k,\rho)\lambda_k - D_1h^T(q_k,\rho)\frac{\partial \lambda_k}{\partial q_k}]
\end{array}
\label{eq-A11}
\end{equation}
\begin{equation}
\frac{\partial q_{k+1}}{\partial p_k} = -M_{k+1}^{-1}
\label{eq-A12}
\end{equation}
\begin{equation}
\begin{array}{l}
\frac{\partial p_{k+1}}{\partial q_k} =\\\hspace{0pt} [D_2^2L_d(q_k,q_{k+1},\rho) + D_2F_d^+(q_k,q_{k+1},\rho,t_k,t_{k+1})]\frac{\partial q_{k+1}}{\partial q_k} \\\hspace{0pt}+ D_1D_2L_d(q_k,q_{k+1},\rho) + D_1F_d^+(q_k,q_{k+1},\rho,t_k,t_{k+1})
\end{array}
\label{eq-A21}
\end{equation}
\begin{equation}
\begin{array}{l}
\frac{\partial p_{k+1}}{\partial p_k} =  \\\hspace{0pt}[D_2^2L_d(q_k,q_{k+1},\rho) + D_2F_d^+(q_k,q_{k+1},\rho,t_k,t_{k+1})]\frac{\partial q_{k+1}}{\partial p_k}
\end{array}
\label{eq-A22}
\end{equation}
\end{subequations}
where $\frac{\partial \lambda}{\partial q_k}$ can be found in \cite{johnson_murphey_linearization}.  Notice the calculations for Eqs.\ (\ref{eq-A21}) and (\ref{eq-A22}) rely on the calculations for Eqs.\ (\ref{eq-A11}) and (\ref{eq-A12}) respectively.  The $B_k$ term is given by chain rule:
\begin{subequations}
\label{eq-B}
\begin{equation}
\begin{array}{l}
\frac{\partial q_{k+1}}{\partial a} = \frac{\partial q_{k+1}}{\partial p_k}\frac{\partial p_{k}}{\partial a} + \frac{\partial q_{k+1}}{\partial q_k}\frac{\partial q_{k}}{\partial a} + M_{k+1}^{-1}[ D_1D_2h^T(q_k,\rho)\lambda_k  \\\hspace{0pt}  - D_3D_1L_d(q_k,q_{k+1},\rho) - D_3F_d^-(q_k,q_{k+1},\rho,t_k,t_{k+1}) ]
\end{array}
\label{eq-B11}
\end{equation}
\begin{equation}
\begin{array}{l}
\frac{\partial p_{k+1}}{\partial a} =\\\hspace{0pt} [D_2^2L_d(q_k,q_{k+1},\rho) + D_2F_d^+(q_k,q_{k+1},\rho,t_k,t_{k+1})]\frac{\partial q_{k+1}}{\partial a} \\\hspace{0pt}+ [D_1D_2L_d(q_k,q_{k+1},\rho) + D_1F_d^+(q_k,q_{k+1},\rho,t_k,t_{k+1})]\frac{\partial q_{k}}{\partial a}  \\\hspace{0pt}+ D_3D_2L_d(q_k,q_{k+1},\rho) + D_3F_d^+(q_k,q_{k+1},\rho,t_k,t_{k+1})
\end{array}
\label{eq-B21}
\end{equation}
\end{subequations}
The term $B_k$ depends on $A_k$ and the previous term $B_{k-1}$. 

We use the equations of motion of discrete mechanical system, Eq.(\ref{eq-fk}), as well as its linearization, Eqs.(\ref{eq-A}) and (\ref{eq-B}) for simulation and calculating the parameter identification gradient for optimal parameter identification.


\section{Parameter Optimization}
\label{sec-opt}
The goal of parameter optimization is to calculate the model parameters $\rho\in\mathcal{P}$ that minimize a cost functional.  For the continuous problem, the cost functional is the integral of a running cost $\ell(x(t),\rho)$ plus a terminal cost $m(x(t_f),\rho)$:
\begin{problem}[Continuous System Parameter Optimization]
Calculate the parameters $\rho\in\mathcal{P}$ which solves:
\[
\min_{\rho\in\mathcal{P}} \Big[J(\rho):=\int_0^{t_f}\ell(x(t),\rho)dt + m(x(t_f),\rho)\Big]
\]
constrained to $\dot{x}(t) = f(x(t),\rho,t)$.
\end{problem}
The gradient and Hessian of the continuous cost is given in \cite{miller_murphey}

For the discrete problem, it is reasonable to choose a discrete cost functional that approximates the continuous cost\textemdash i.e. $\ell_d(x_{k},\rho)\approx\int_{t_{k-1}}^{t_{k}}\ell(x(\tau),\rho)d\tau$ and $m_d(x_{k_f},\rho)\approx m(x(t_f),\rho)$.  Alternatively, $\ell_d$ and $m_d$ can be designed directly without first choosing an underlying continuous cost.  The discrete parameter optimization problem is as follows:
\begin{problem}[Discrete System Parameter Optimization]
Calculate the parameters $\rho\in\mathcal{P}$ which solves:
\[
\min_{\rho\in\mathcal{P}} \Big[J_d(\rho):=\sum_{k=1}^{k_f}\ell_d(x_k,\rho) + m_d(x_{k_f},\rho)\Big]
\]
constrained to $x_{k+1} = f(x_k,\rho,t_k)$, Eq.(\ref{eq-fk}).
\label{prob-disc}
\end{problem}

In optimal control theory, it is common practice to solve optimization problems using iterative methods.  Iterative optimization methods repeatedly reduce the cost by stepping in a descending direction until a local optimum is found.  Commonly, the step direction and step size is calculated using local derivative information \cite{armijo, kelley}, which is practiced in this paper.  In the next section, we provide an adjoint-based calculation for the gradient of the cost functional with respect to the parameters.

\subsection{Discrete System Parameter Gradient}
The gradient of the cost functional given in problem \ref{prob-disc} is provided in the following lemma.
\begin{lemma}
\label{lem-grad_a}
Suppose $L_d(q_k,q_{k+1},\rho)$, $F_d^-(q_k,q_{k+1},\rho,t_k,t_{k+1})$, $F_d^+(q_k,q_{k+1},\rho,t_k,t_{k+1})$, and $h(q_k,\rho)$ are $\mathcal{C}^2$ with respect to $q_k$, $q_{k+1}$ and $\rho$.  Take $A_k$ and $B_k$ from Eq.(\ref{eq-lin_fk}) and assume $M_k$ is always nonsingular.  Then,\footnote{The notation $D$ is a slot derivative for a function of only one argument.  For example, $DJ_d(\rho)$ is the partial derivative of $J_d$ with respect to $\rho$.}
\begin{equation}
DJ_d(\rho) = \sum_{k = 1}^{k_f}\lambda_kB_{k-1} +D_2\ell_d(x_k,\rho) + D_2m_d(x_{k_f},\rho)
\label{eq-DJa}
\end{equation}
where $\lambda_k$ is the solution to the backward one-step mapping
\begin{equation}
\lambda_k = \lambda_{k+1}A_{k} + D_1\ell_d(x_{k},\rho) 
\label{eq-lambda}
\end{equation}
starting from $\lambda_{k_f} = D_1\ell(x_{k_f},\rho) + D_1m_d(x_{k_f},\rho)$.  
\end{lemma}
\begin{proof}
The derivative of the cost in the direction $\theta\in\setR^{n_\rho}$ is
\begin{equation}
\begin{array}{l}
DJ_d(\rho)\theta:=\sum_{k=1}^{k_f}D_1\ell_d(x_k,\rho)\frac{\partial x_k}{\partial \theta} + D_2\ell_d(x_k,\rho)\theta \\\hspace{20pt}+ D_1m_d(x_{k_f},\rho)\frac{\partial x_{k_f}}{\partial \theta} + D_2m_d(x_{k_f},\rho)\theta.
\end{array}
\label{eq-DJ_dot_theta}
\end{equation}
Label $z_k:=\frac{\partial x_{k}}{\partial \theta}$ for convenience. Also for convenience, label 
\begin{equation}
H:=\sum_{k=1}^{k_f}D_1\ell_d(x_k,\rho) z_k + D_1m_d(x_{k_f},\rho) z_{k_f}
\label{eq-H}
\end{equation}
The linearized state, $z_k$, is the solution to the linearized state equation, Eq.(\ref{eq-lin_fk}).  In other words,
\[
z_{k+1} = A_kz_k + B_k\theta
\]
starting from $z_0 = 0$.  The linearized state's solution depends on the discrete state transition matrix:
\[
\Phi(k_2,k_1):=\prod_{j = 1}^{k_2-k_1} A_{k_2-j} = A_{k_2-1}A_{k_2-2}\cdots A_{k_1}
\]
for integers $k_2>k_1$ and where $\Phi(k_1,k_1):=I$, the identity matrix.  Recalling $z_0 = 0$, the linearized state's solution is
\[
z_k = \Phi(k,0)z_0 + \sum_{s = 0}^{k-1}\Phi(k,s+1)B_s\theta = \sum_{s = 0}^{k-1}\Phi(k,s+1)B_s\theta
\]
for $k = 1,\ldots,k_f$.  Plugging $z_k$ into $H$, Eq.(\ref{eq-H}), $H$ becomes
\[
\begin{array}{l}
H = \sum_{k=1}^{k_f}D_1\ell_d(x_k,\rho) \sum_{s = 0}^{k-1}\Phi(k,s+1)B_s\theta \\\hspace{20pt}+ D_1m_d(x_{k_f},\rho) \sum_{s = 0}^{k_f-1}\Phi(k_f,s+1)B_s\theta
\\\hspace{10pt} = \sum_{k=1}^{k_f}\sum_{s = 0}^{k-1}D_1\ell_d(x_k,\rho) \Phi(k,s+1) B_s\theta \\\hspace{20pt}+  \sum_{s = 0}^{k_f-1} D_1m_d(x_{k_f},\rho)\Phi(k_f,s+1)B_s\theta
\end{array}
\]
Switch the order of the double sum.  
\[
\begin{array}{l}
H = \sum_{s=0}^{k_f-1}\sum_{k = s+1}^{k_f}D_1\ell_d(x_k,\rho) \Phi(k,s+1) B_s\theta \\\hspace{20pt}+ D_1m_d(x_{k_f},\rho) \sum_{s = 0}^{k_f-1}\Phi(k_f,s+1)B_s\theta
\\\hspace{10pt} = \sum_{s=0}^{k_f-1}\Big[\sum_{k = s+1}^{k_f}D_1\ell_d(x_k,\rho) \Phi(k,s+1) \\\hspace{20pt}+ D_1m_d(x_{k_f},\rho)\Phi(k_f,s+1)\Big]B_s\theta
\end{array}
\]
Set $\lambda_{s+1}$ as the resulting co-vector in the brackets so that 
\[
H = \sum_{s=0}^{k_f-1} \lambda_{s+1} B_s\theta = \sum_{k=1}^{k_f} \lambda_{k} B_{k-1}\theta.
\]
The efficient calculation for $\lambda_k$ is given in Eq.(\ref{eq-lambda}).  Plugging $H$ into Eq.(\ref{eq-DJ_dot_theta}), we find
\[
DJ(\rho)\theta = \Big[\sum_{k = 1}^{k_f}\lambda_kB_{k-1} +D_2\ell_d(x_k,\rho) + D_2m_d(x_{k_f},\rho)\Big]\theta.
\]
\end{proof}

\section{Experiment}
\label{sec-experiment}


\begin{figure*}
\centering
\def\svgwidth{.97\textwidth}
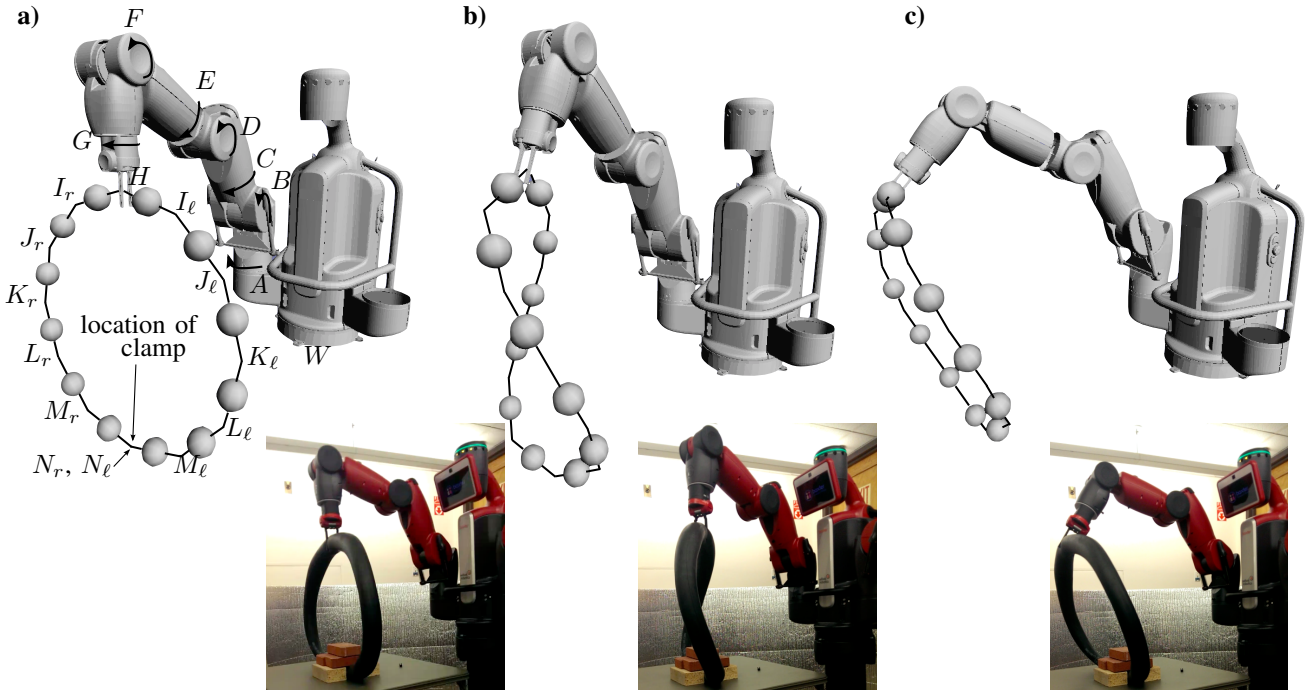
\caption{Three distinct configurations along with snapshots of the physical system.  \textbf{a)} The frames for Baxter and the loop in their initial configuration. \textbf{b)} Baxter ``twisting'' the loop. \textbf{c)} Baxter ``bending'' the loop.}
\label{fig-3bloops}
\end{figure*}

The goal is to identify the model parameters of a flexible loop that best match experimental measurements of a flexible loop with only touch at a single point of contact.  For the measuring and manipulating, we use Rethink Robotics' Baxter robot, shown in Figure~\ref{fig-baxter_image_1}. We simulate this interaction using a discrete model of the loop shown in Figure \ref{fig-3bloops}. The simulated arm joint angles and torques depend on the parameters set for a simulation.  As such, we can find the parameters which correspond to simulated joint angles that best match the measured joint angles.  For this experiment, the unknown parameters describe the stiffness of the loop model.

The tools for simulating Baxter manipulating a loop are given in Section \ref{sec-sys} while the optimization technique for best matching the simulation to the measured data is described in Section \ref{sec-opt}.  The following provides the experimental setup, the model and simulation of Baxter manipulating the loop, and the results of optimally identifying loop parameters.

\subsection{Experimental Setup \label{sec-setup}}

The experiment is set up as follows.  Baxter's arm is positioned similarly to that in Figure~\ref{fig-baxter_image_1}.  The top of an inflated rubber loop (bicycle tyre) is placed in Baxter's gripper and the bottom is clamped to the table.   The loop is initially positioned vertically. The loop has a radius of $0.355$ meters and a mass of $0.132$ kilograms.  The arm is run through a recorded motion that is designed to stretch, bend, and twist the tube. The arm's joint angles and torques are captured every 100Hz, which we compile into the two vectors $b_{meas}(t)$ and $T_{meas}(t)$.  The measured joint angles have a position resolution of +/- 5 mm.

The measured joint angles $b_{meas}$ correspond to a subset of the system configuration $q$, where the remaining configurations belong to the loop.  We label this subset of $q$ that describes the arm as $b$.  The optimization goal is to choose model parameters so that $|b-b_{meas}|$ is small.

\subsection{Model}
We model Baxter and the loop with the same underlying rigid body mechanics.  We model both as a series of rigid links connected by rotational joints.  For Baxter, these joints are 7 series elastic actuators.  The dimensions,  inertia, and other information concerning Baxter's arm can be obtained at \url{https://github.com/RethinkRobotics}.  

\begin{figure*}[!htb]
\centering
\def\svgwidth{.80\textwidth}
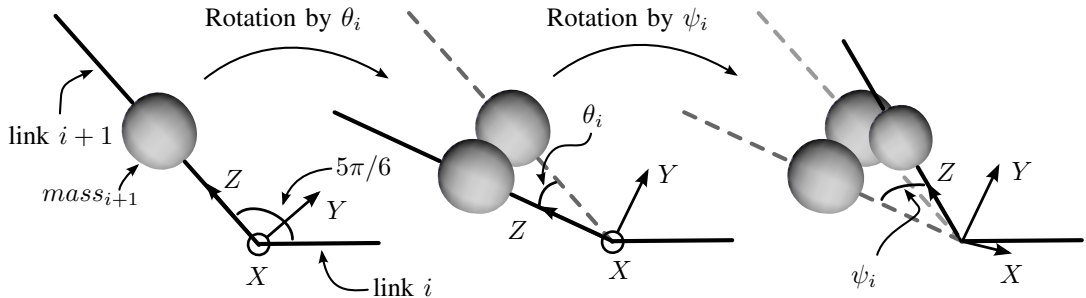
\caption{Illustration of joint $i$ connecting loop links $i$ and $i+1$.  The joint is a rotation of $\theta_i$ radians about the $X$-axis followed by a rotation of $\psi_i$ radians about the $Y$-axis. }
\label{fig-loop_link}
\end{figure*}

We use a discrete model of the loop.  The loop is composed of 12 rigid links wherein the connection between each rigid link is a spherical joint.  A single loop joint is shown in Figure~\ref{fig-loop_link}.  As seen in the figure, the orientation of the new link is given by the angles $\theta_i$ and $\psi_i$.  These rotations are described on a frame that is rotated about the $X-axis$ by $5\pi/6$ radians so that when each $\psi_i = \theta_i = 0$, the loop approximation is a regular dodecagon.  These angles, the $\theta_i$ and $\psi_i$, along with Baxter's joint angles, $b$, make up the system configuration, $q$. Note that the choice of the discretization is arbitrary and only dependent on the required fidelity of the model as well as available computational time. 

We represent Baxter grasping the loop using a tree structure of transformations.  This tree structure is illustrated in Figure~\ref{fig-3bloops}a).  Starting from the world frame $\mathcal{W}$, the arm is specified with successive rotation and translation transformations.  At Baxter's end effector, labeled $H$, the description branches, splitting the loop into two sides, marked with $\ell$ and $r$ for left and right.  The two branches meet at the base of the loop, where the loop is clamped.  For arbitrary Baxter and loop joint angles, $q$, there is no guarantee that the two ends of the loop meet at the clamping location in space.  When they do, we say that the system configuration satisfies the holonomic constraint $h(q) = 0$.  There are 12 total constraints---i.e. $h(q)\in\setR^{12}$.  There are 6 constraints for each side of the loop transformation description in order to constrain the orientation and position of frames $N_\ell$ and $N_r$ to the clamp location.  Figure~\ref{fig-3bloops} shows three distinct configurations $q$ that satisfy the constraints.

We model the stiffness of the loop with torsional springs on each loop configuration variable---i.e. for each $\theta_i$ and $\psi_i$.  Due to the uniformity of the loop, we assume that all of the springs on the $\theta_i$ configuration variables have the same spring constant $\kappa_{\theta}$.  Likewise, each of the springs on $\psi_i$ configuration variables have same spring constant $\kappa_{\psi}$.  The objective is to identify these spring constant and as such we set $\rho = [\kappa_{\theta}, \kappa_{\psi}]^T$, where $\mathcal{P} = \setR^+ \times \setR^+$.  

\subsection{Simulation}
In order to simulate the robot arm and loop, we turn to Section \ref{sec-sys} which reviews variational integrators.  The continuous state is given by $q$ and its time derivatives\textemdash i.e. $x(t) = [q(t),\dot{q}(t)]^T$.  These are the robot and loop joint angles and associated angular velocities.  The system's discrete time state is $x_k = [q_k,p_k]^T$ where $p_k$ is defined in Eq.(\ref{eq-pk}).   The continuous dynamics are given by the constrained, forced Euler Lagrange equations, Eq.(\ref{eq-cont_EL}), which depends on the system Lagrangian $L(q,\dot{q},\rho)$, external forces $F_c(q,\dot{q},\rho,t)$ and holonomic constraints $h(q,\rho)$. Techniques to derive these formulas for rigid bodies are well understood \cite{murray_li_sastry}.  

From the continuous dynamics, it is straightforward to obtain the discrete dynamics, which are given in Eq.(\ref{eq-disc_EL}).  The discrete dynamics depend on the discrete system Lagrangian $L_d(q_{k},q_{k+1},\rho)$, discrete external left and right forces, $F_d^-(q_{k},q_{k+1},\rho,t_k,t_{k+1})$ and $F_d^+(q_{k},q_{k+1},\rho,t_k,t_{k+1})$, and holonomic constraints $h(q_k,\rho)$.  

For simulation, we chose a constant time step of $\Delta_t = 0.01$ seconds, which matches the broadcast frequency of Baxter.  Further, we decided to approximate $L_d$ from $L$ using midpoint rule (see Eq.(\ref{eq-Ld})):
\begin{equation}
L_d(q_k,q_{k+1},\rho) = \Delta_t L(\frac{q_{k+1}+q_k}{2},\frac{q_{k+1}-q_k}{\Delta_t},\rho).
\label{eq-Ld_midpoint}
\end{equation}
Similarly, using midpoint rule, we approximate $F_c$ by $F_d^-$ and $F_d^+$ where
\[
\left\{
\begin{array}{l}
F_d^-(q_k,q_{k+1},\rho,t_k,t_{k+1}) = \\\hspace{70pt}\Delta_t F_c(\frac{q_{k+1}+q_k}{2},\frac{q_{k+1}-q_k}{\Delta_t},\rho,\frac{t_{k+1}+t_k}{2})\\
F_d^+(q_k,q_{k+1},\rho,t_k,t_{k+1}) = 0.
\end{array}
\right.
\]
It is a simple process to translate the continuous system to the discrete variational integrator one. Additionally, in order to simulate the dynamics using the one-step mapping in Eq.(\ref{eq-fk}), we need certain partial derivatives of the Lagrangian, external forces and constraints with respect to configuration variables, which can be found in \cite{johnson_murphey_scalable}.  

The external forces $F_c$ are applied at the configuration variables.  For Baxter manipulating the loop, the only external forces are those applied at the 7 configuration variables that define Baxter's arm, which we have labeled as $b$. Due to model and sensor error, which are always an issue for real systems, it is unreasonable to expect that the simulation will provide meaningful results, let alone be stable by directly setting $F_c = T_{meas}$, the experimentally measured joint torques.  Instead, we decided to filter them using a proportional feedback loop as so:
\[
F_{c}(t) = T_{meas}(t) - K(t) (b(t) - b_{meas}(t)),
\]
where $b_{meas}$ is Baxter's measured joint angles as mentioned in Section \ref{sec-setup} and $K$ is a feedback gain.  When $K$ is chosen correctly, the simulation is stable.  We chose $K$ from a finite horizon LQR to calculate an optimal feedback gain from the model linearized around $b_{meas}$ and a quadratic cost functional \cite{anderson_moore}.

\emph{Aside:}  We used the software tool \texttt{trep} \cite{johnson_murphey_scalable} which simulates articulated rigid bodies using midpoint variational integrators.  It additionally provides partial derivative calculations that we need for the system linearization, Eqs.(\ref{eq-A}) and (\ref{eq-B}).  

\subsection{Linearization}
The linearization of the discrete equations of motion is given by matrices $A_k$ and $B_k$ in Eqs.(\ref{eq-A}) and (\ref{eq-B}).  We need the linearization for the gradient calculation, Lemma \ref{lem-grad_a}, in order to perform gradient-based descent algorithm like steepest descent for parameter identification.  Partial derivatives of $L_d$ and $F^+$ with respect to $q_k$ and $p_k$ can be obtained from \cite{johnson_murphey_linearization}.  In this section, we only concern ourselves with the partials that depend on the paramaters $\rho = [\kappa_\theta,\kappa_\psi]^T$ which are not included in \cite{johnson_murphey_linearization}

For the loop example, we need to calculate $D_3D_1L_d(q_k,q_{k+1},\rho)$ and $D_3D_1L_d(q_k,q_{k+1},\rho)$ for $\rho = [\kappa_\theta, \kappa_\psi]$.  Note that the potential energy of the system can be written as:
\[
V(q,\rho) = V_{\theta}(q,\kappa_{\theta}) + V_{\psi}(q,\kappa_{\psi}) + V_g(q)
\]
where $V_{\theta}$, $V_{\psi}$ and $V_g$ are the potential energies due to the spring torques on the $\theta_i$ configuration variables, the spring torques on the $\psi_i$ configuration variables, and gravity, respectively.  Label $I_\theta$ and $I_\psi$ as the index of the $\theta_i$ and $\psi_i$ configuration variables in $q$ respectively. The potential energy due to the $\theta_i$ torsional springs is $V_{\theta}(q, \kappa_{\theta}) = \sum_{i\in I_\theta}\frac{1}{2}\kappa_{\theta}\theta_i^2$.  Approximating for the discrete time potential energy---see Eq.(\ref{eq-Ld_midpoint})---we find that \footnote{Here, we (poorly) index the $i^\textrm{th}$ term of $q_k$ as $q_{i,k}$.}
\[
V_{\theta,d}(q_k,q_{k+1},\kappa_{\theta}) = \sum_{i \in I_\theta}\frac{\Delta_t}{2}\kappa_{\theta}(\frac{q_{i,k+1} + q_{i,k}}{2})^2.
\]
Taking the needed partial derivatives to calculate $D_3D_1L_d$, the $i^\textrm{th}$ element of $D_3D_1V_{\theta,d}$ is
\[
\begin{array}{l}
D_3D_1 V_{\theta,d}(q_k,q_{k+1},\kappa_{\theta})_i 
= \left\{\begin{array}{ll}
\frac{\Delta_t}{4}(q_{i,k+1} + q_{i,k}),   &     i\in I_\theta \\
0 & \textrm{else}.
\end{array}\right.
\end{array}
\]
The needed partial derivatives of $V_{\psi,d}$ are the same except for the indexes $I_\psi$
Since the kinetic energy does not depend on $\rho = [k_\theta, k_\psi]$,
\[
\begin{array}{l}
D_3D_1L_d(q_k,q_{k+1},\rho) = \\\hspace{10pt}- [D_3D_1 V_{\theta,d}(q_k,q_{k+1},\kappa_{\theta}), D_3D_1 V_{\psi,d}(q_k,q_{k+1},\kappa_{\psi})].
\end{array}
\]
Repeating the derivation for $D_3D_2L_d(q_k,q_{k+1},\rho)$ we find that $D_3D_2L_d(q_k,q_{k+1},\rho) = D_3D_1L_d(q_k,q_{k+1},\rho)$

Furthermore, we need to calculate $D_1h(q_k,\rho)$, $D_1^2h(q_k,\rho)$ and $D_1D2h(q_k,\rho)$, the last of which is $0$ since the constraints do not depend on the parameters.  These partial derivatives of $h$ are given simply by chain rule and depend on the first and second partial derivatives of the transformations from the world frame $W$ to frames $N_r$ and $N_\ell$ (refer to Figure~\ref{fig-3bloops} for the frames).  These partial derivatives can be found in \cite{johnson_murphey_linearization}.  

\subsection{Optimal Parameter Identification}
With the linearization of the discrete state equations, Eq.(\ref{eq-fk}), we can calculate the gradient of a cost from Lemma \ref{lem-grad_a} in a steepest descent algorithm to identify the model parameters $\rho = [\kappa_\theta,\kappa_\psi]^T$.  The experiment is set up as in Section \ref{sec-setup}.  We program the experimental Baxter arm to stretch, bend and twist the loop over a 20 second time frame. Snapshots of this manipulation are in Figure~\ref{fig-3bloops}.  The joint torques to displace the loop are recorded in $T_{meas}$ while joint angles are recorded in $b_{meas}$.    For fixed parameters $\rho$, the system can be simulated as discussed in Section \ref{sec-sys} using $T_{meas}$.  We wish to find the simulated Baxter arm joint angles, labelled $b$---recall $b$ are the configuration variables of $q$ that describe the Baxter arm---that best match $b_{meas}$.

The matching is quantified by the cost $J_d$.  We choose $J_d$ to be quadratic on the error from the simulated end effector position in space to the measured end effector position, which can be derived from $b$ and $b_{meas}$ respectively.  For reference, the position of the end effector is at the origin of frame $H$ in Figure~\ref{fig-3bloops}a.  Label $w_k(\rho)$ as this simulated position at time $t_k$ for parameters $\rho$ and $w_{meas}(t_k)$ as this measured position.  Set $\epsilon_k := w_k(\rho)-w_{meas}(t_k)$.  The cost $J_d$ is defined by the running cost $\ell_d(q_k,\rho)$ and the terminal cost $m_d(q_{k_f},\rho)$, set as
\[
\ell_d(q_k,\rho) = \epsilon_k^T\epsilon_k \textrm{ and } m_d(q_{k_f},\rho) = \epsilon_{k_f}^T\epsilon_{k_f}
\]
We perform the optimization using a steepest descent algorithm with the inequality constraint that the parameters are $\kappa_\theta>0$ and $\kappa_\psi>0$.    At each iteration of the descent, an Armijo line search updates the parameters by choosing a distance to step in the direction of the negative gradient.  We used Armijo parameters $\alpha = \beta = 0.4$ \cite{armijo}.

We seed the steepest descent algorithm with an initial guess of $\rho = [5, 5]^T$.  After $75$ iterations, the algorithm terminates with gradient norm $|DJ_d(\rho)| = 0.001138$.  The cost decreased from $J_d = 2.13506$ to $1.106741$.  The parameters are identified as $\rho^\star = [4.45252, 0.96969]^T$.  The convergence is shown in Fig.~\ref{fig-conv}.
Furthermore, a comparison of the simulated end effector's path for each iteration of steepest descent with the Baxter measured end effector's path is in Fig.~\ref{fig-paths}.  In agreement that the cost decreases with each iteration, it appears in the figure that the end effector's path converged toward the measured end effector's path.

For the example, optimally identifying a 12-link  (31 configuration, 62 state) model from 20 seconds of manipulation activity (2000 data points recorded at 100Hz) takes 219.05 minutes on a Macbook air.  For comparison, we also identified a 6-link (19 configuration 38 state) model.  It converged nearly four times quicker, in 61.01 minutes, but the optimal cost was higher at $J_d = 1.54588$.  The design tradeoff is between model fidelity and computation time.
\begin{figure}
\centering
\includegraphics[width = 160pt]{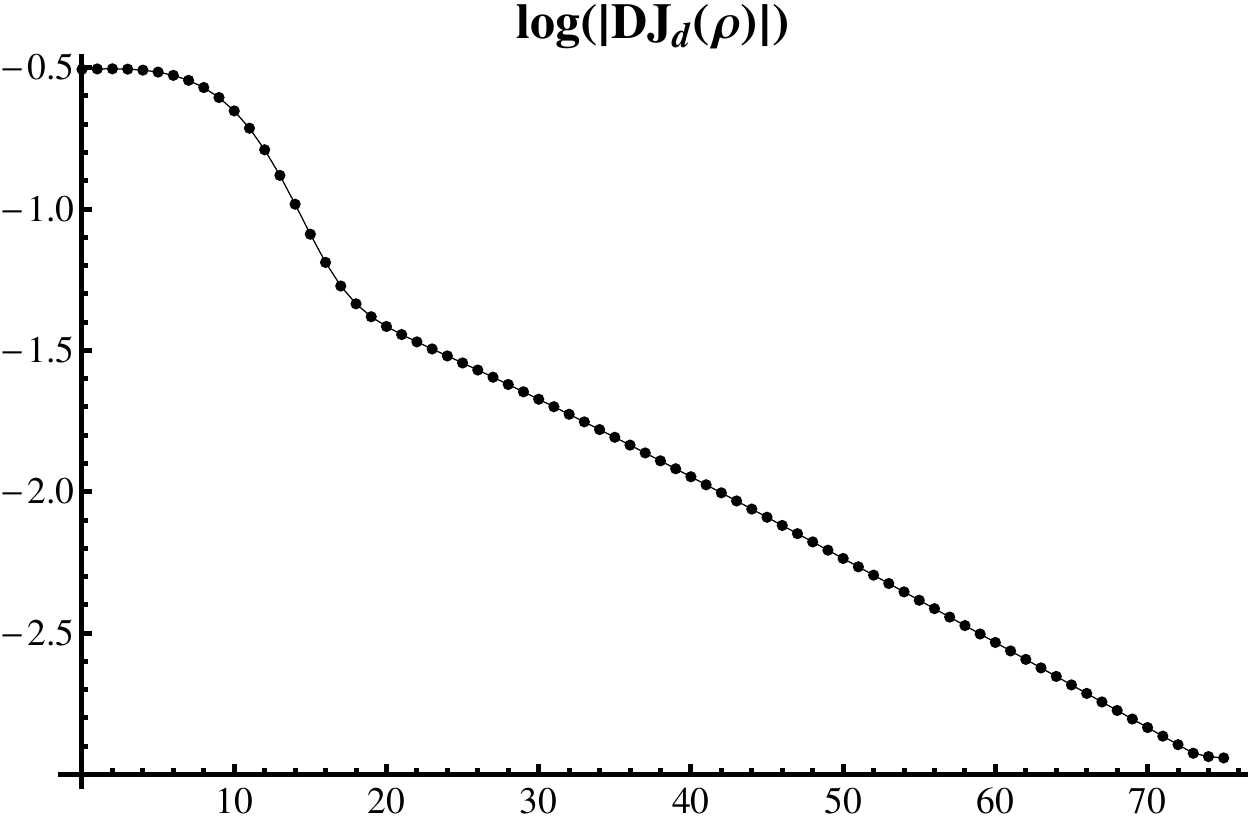}
\caption{Convergence of the optimization algorithm.}
\label{fig-conv}
\end{figure}

\begin{figure}
\centering
\includegraphics[width = 160pt]{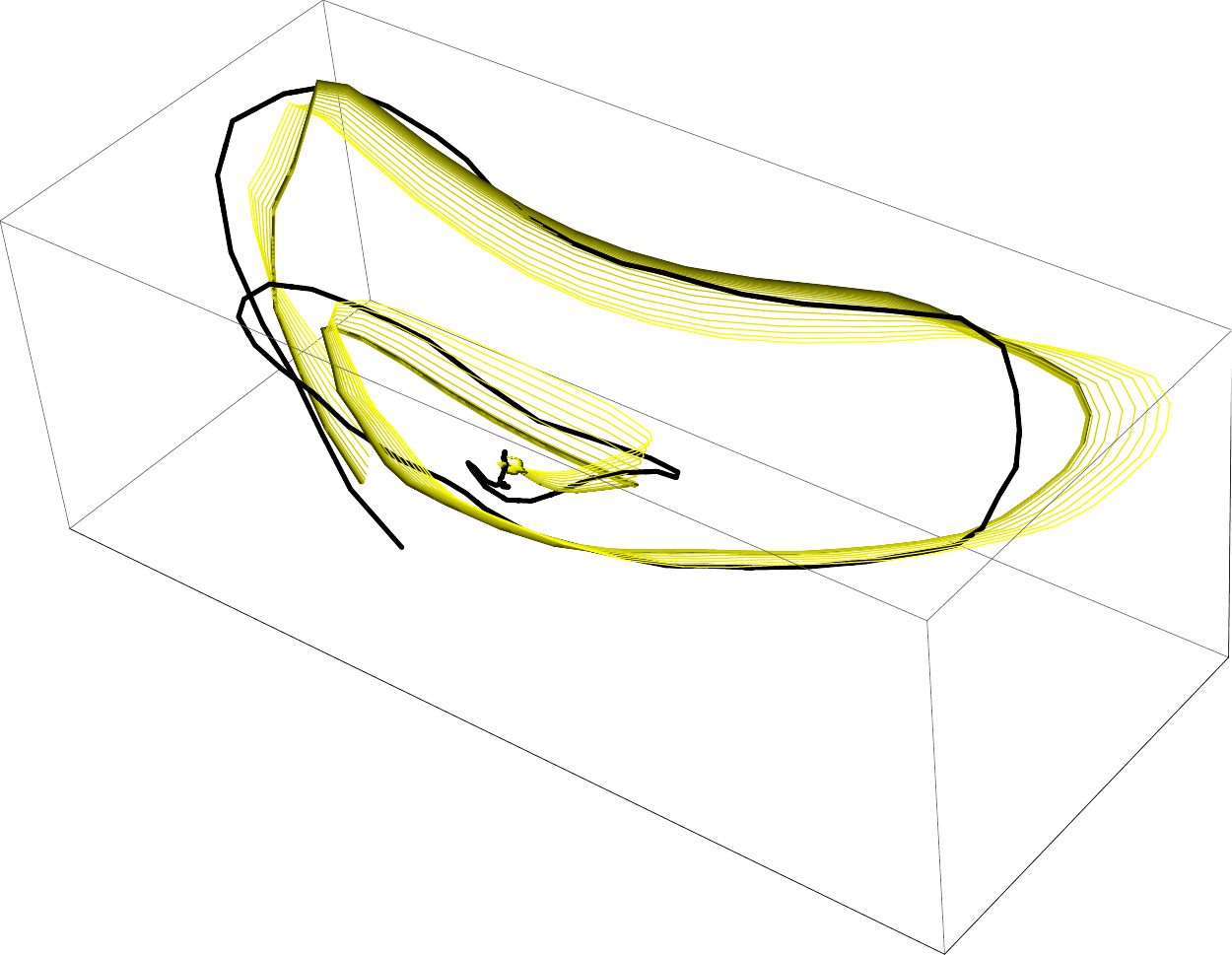}
\caption{The path the robot's end effector took through space.  The black line is the measured end effector's path while the yellow lines are the end effector's simulated path for each iteration of the steepest descent algorithm.  The iteration numbers are ordered from lighter yellow to darker yellow.  }
\label{fig-paths}
\end{figure}

\section{Discussion}
Through the proposed parameter identification procedure, we calculated the model parameters that best match physical phenomena within the constraint of the chosen model.  This process is important since an improved model can make for better object manipulation.  However, it is unclear, especially in the presented experiment, what this matching tells us about the object's physical properties.  Certainly, measuring at a single contact point provides little insight into the interior stress and strain of the rubber loop.

Furthermore, we assumed the loop has uniform stiffness.  Without this assumption, multiple experiments are needed at different contact points to identify the non-uniformity.  Also, objects with more complex geometries require additional experiments.  For instance, grasping and manipulating a single leaf of a tomato plant provides little insight to identify the full plant.  Multiple experiments as well as algorithms to determine appropriate actions is a topic of future work.

As the parameter identification process is informed from actual manipulation, the proposed method could also be used online, thereby improving a model estimate with the time the robot spent with an object growing. We note that identification in this paper relies exclusively on proprioceptive data compared with model prediction. In the future, we plan to explore the following directions: (1) using proprioception of a second arm to validate hypotheses on object movement and obtaining additional measurements and constraints, and (2) combine this data with exterioperception such as cameras, 3D data, and dynamic tactile sensing \cite{hughes2014}. 

Finally, we note that the fidelity of the approach heavily depends on the choice of the model. While this choice strongly depends on the desired manipulation task, we are also interested in automatically finding appropriate model representations for given geometries and conceptual knowledge on the object, such as ``plant'', ``tube'' or ``sheet'', e.g.  

\section{Conclusion}
We provided a method for optimally identifying the model parameters of a flexible object manipulated by a robotic arm and applied it to identifying stiffness characteristics of a flexible loop. The proposed model exclusively relies on proprioception, thereby not requiring any additional sensors. We modeled the loop as a chain of rigid links connected by torsional springs and used variational integrators to simulate it.  The simulation accounted for the chain of links being closed to form the loop. Additionally, since we modeled the loop with the same underlying mechanics as the robotic arm, the full system can be simulated together, allowing for planning in the configuration space.  The feasibility of the approach was demonstrated using the variational integrator simulating software \texttt{trep} and with data recorded from a 7-DOF series elastic robot arm. We have identified open challenges, including using multiple arms and external sensors to increase model performance, as well as identifying not only model parameters, but also the model structure itself. 

\bibliographystyle{plain}
\bibliography{param_opts_cites}

\end{document}